\documentclass{article}
\usepackage{ijcai17}

\usepackage{times}
\usepackage{graphicx}
\usepackage{amsmath, amssymb, amsthm}
\usepackage{mathrsfs}
\usepackage{bm}
\usepackage{lipsum}
\usepackage[scientific-notation=true]{siunitx}
\usepackage{multirow}
\usepackage{enumitem}
\usepackage{pifont}
\usepackage[skip=2pt]{caption}
\usepackage{algorithm}
\usepackage{algorithmic}
\renewcommand{\algorithmiccomment}[1]{\bgroup\hfill\tiny\#~#1\egroup}
\graphicspath{ {figs/} }
\usepackage{mathtools}

\newtheorem{definition}{Definition}
\newtheorem{theorem}{Theorem}
\newtheorem{lemma}[theorem]{Lemma}
\newtheorem{problem}{Problem}%[section] % Comment out [section] to 
\newtheorem{case}{Case}

\newsavebox\IBoxA \newsavebox\IBoxB \newlength\IHeight
\newcommand\TwoFig[6]{% Image1 Caption1 Label1 Image2 ...
  \sbox\IBoxA{\includegraphics[width=0.48\textwidth]{#1}}
  \sbox\IBoxB{\includegraphics[width=0.48\textwidth]{#4}}%
  \ifdim\ht\IBoxA>\ht\IBoxB
    \setlength\IHeight{\ht\IBoxB}\else\setlength\IHeight{\ht\IBoxA}\fi%
  \begin{figure*}[!htb]
  \minipage[t]{0.48\textwidth}\centering
  \includegraphics[height=\IHeight]{#1}
  \caption{#2}\label{#3}
  \endminipage\hfill
  \minipage[t]{0.48\textwidth}\centering
  \includegraphics[height=\IHeight]{#4}
  \caption{#5}\label{#6}
  \endminipage 
  \end{figure*}%
}

\begin{document}

\title{Adversarial Generation of Real-time Feedback with Neural Networks \\for Simulation-Based Training}

% \title{Adversarial Real-time Feedback for Simulation-based Training}

% with \\
% Neural Networks for Simulation-based Training}

\author{Xingjun Ma, Sudanthi Wijewickrema, Shuo Zhou, Yun Zhou, \\
{\bf Zakaria Mhammedi, Stephen O'Leary, James Bailey} \\
The University of Melbourne, Australia  \\
\{xingjunm@student, swijewickrem@, zhous@student, yun.zhou@, \\ 
zmhammedi@student., sjoleary@, baileyj@\}unimelb.edu.au
}

\maketitle

\begin{abstract}
Simulation-based training (SBT) is gaining popularity as a low-cost and convenient training technique in a vast range of applications. However, for a SBT platform to be fully utilized as an effective training tool, it is essential that feedback on performance is provided automatically in real-time during training. It is the aim of this paper to develop an efficient and effective feedback generation method for the provision of real-time feedback in SBT. Existing methods either have low effectiveness in improving novice skills or suffer from low efficiency, resulting in their inability to be used in real-time. In this paper, we propose a neural network based method to generate feedback using the adversarial technique. The proposed method utilizes a bounded adversarial update to minimize a $L1$ regularized loss via back-propagation. We empirically show that the proposed method can be used to generate simple, yet effective feedback. Also, it was observed to have high effectiveness and efficiency when compared to existing methods, thus making it a promising option for real-time feedback generation in SBT.
\end{abstract}

\section{Introduction}
\label{sec:introduction}
Supporting the learning process through interactive feedback is important \cite{billings2012efficacy}. Appropriate and timely feedback intervention increases learning motivation, facilitates skill acquisition/retention, and reduces the uncertainty of how a student is performing \cite{davis2005interactive}. With the development of virtual reality techniques, simulation-based training (SBT) has become an effective training platform in a range of applications including surgery \cite{wijewickremadesign,masimulation}, military training \cite{cosma2011implementing}, and driver/pilot training \cite{de2011effect}. However, it still requires the presence of human experts so that real-time feedback can be provided during training to ensure that relevant skills are learned. This has been one of the obstacles to the spread of SBT systems \cite{lateef2010simulation}. As such, it is important to automate the generation of real-time feedback in SBT.

Feedback generation is a classical problem in artificial intelligence (AI) systems. Intelligent tutoring systems are one such class of AI systems that aims to provide immediate instruction or feedback to learners \cite{billings2012efficacy}. Another example is autonomous driving systems that take the surrounding environment as input and output feedback to the car to adjust the steering wheel \cite{chen2015deepdriving}. In reinforcement learning systems such as mobile robot navigation, the hardware or software agent learns its behaviour based on reward feedback from the environment \cite{sutton1998reinforcement}. 

When compared to the above mentioned applications, SBT focuses more on educational gains such as the acquisition of proper skills \cite{steadman2006simulation,lateef2010simulation}. As such, SBT requires a higher degree of ``hands-on" experiential interaction. Figure \ref{fig:tbs} shows an example of such a SBT system: The University of Melbourne Virtual Reality Temporal Bone Surgical Simulator \cite{wijewickrema2016provision}. Rule-based feedback tutoring methods that work in domains such as algebra and physics are not flexible enough for SBT that requires a high level of user interaction and complex user behaviour. Autonomous driving systems and reinforcement learning systems mainly focus on the outcomes and mostly deal with cognitive tasks. Therefore, feedback generation methods in these systems are not directly transferable to SBT, especially for non-cognitive SBT scenarios. The aim of this paper is to develop an automated feedback generation method that can be used in SBT via supervised learning.

\begin{figure}[!htb]
\centering
\includegraphics[width=0.48\textwidth]{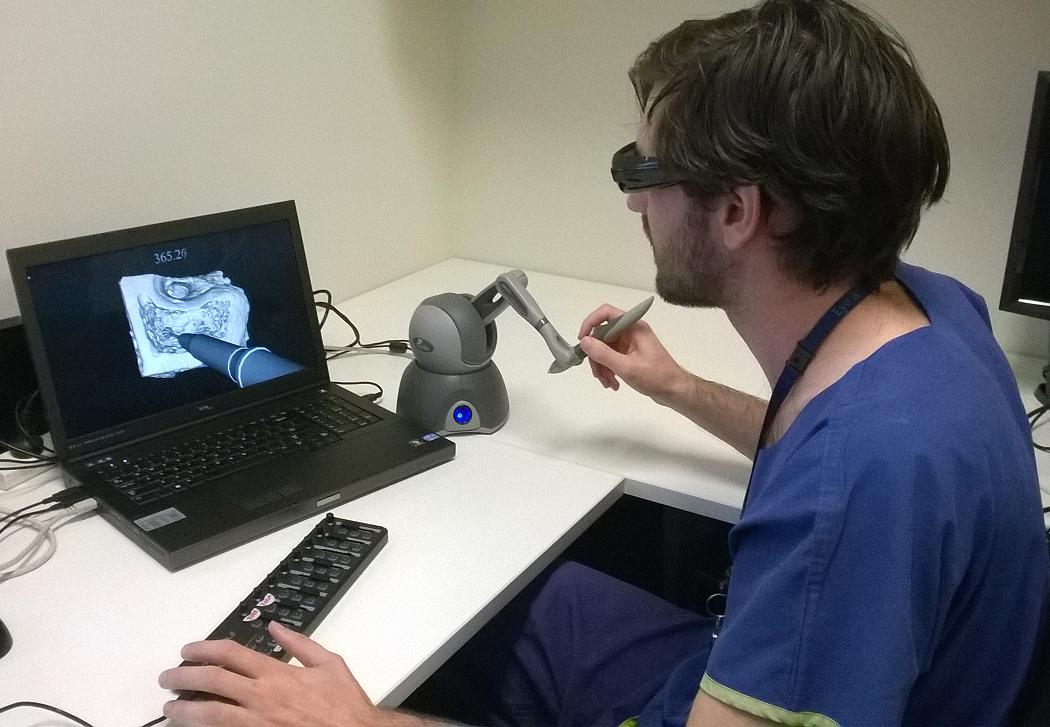}
\caption{\textit{The University of Melbourne Virtual Reality Temporal Bone Surgery Simulator}: it consists of a computer that runs a 3D model of a human temporal (ear) bone and a haptic device that provides tactile resistance to simulate drilling.}
\label{fig:tbs}
\end{figure}

Feedback generation in SBT has three challenges. First, feedback should be generated in a timely manner as delayed feedback can lead to confusion or even cause fatal consequences in reality. An acceptable time-limit is 1 second after inappropriate action is detected. This is because feedback should be provided before the learner makes the next move \cite{rojas2014impact}. Second, feedback should be actionable instructions that can be followed by the trainee to improve skills or correct mistakes. This is because SBT tasks often consist of a series of delicate operations that require precise instructions. Third, feedback should be simple, referring to only a few aspects of the skill, as practically people cannot focus on many things at a time. Also, this reduces distractions to the trainee and decreases cognitive load, thus increasing the usefulness of the feedback \cite{sweller1988cognitive}.

In this paper, we make the following contributions: 
\begin{itemize}
    \item We demonstrate how the adversarial technique can be used to generate actionable knowledge or feedback with neural networks.
    \item We propose a novel neural network based feedback generation method that works with a $L1$ regularized loss to control the simplicity of feedback and a bounded update to ensure the generated feedback has practical meaning. 
    \item We show that the proposed method has high effectiveness as well as high efficiency when compared to existing methods, making it possible to be used for real-time feedback generation in SBT.
\end{itemize}

The structure of the paper is as follows. Section 2 introduces related work in this field. Section 3 illustrates the real-time feedback process and the formal definition of the problem. The proposed feedback generation method is described in Section 4 and evaluated along with existing methods in Section 5. Section 6 concludes the paper.

\section{Related Work}
\label{sec:related}

The simplest way to provide feedback in SBT is the rule-based approach. The ``follow-me" approach (ghost drill) \cite{rhienmora2011intelligent} and the ``step-by-step" approach \cite{wijewickrema2016provision} in surgical simulation are examples of this approach. However, it may be hard for a novice who has limited experience to follow a ghost drill at his own pace, and step-by-step feedback will not respond if the trainee does not follow the suggested paths.

Other works utilize artificial intelligence techniques to generate feedback that can change adaptively in response to the novice's abilities and skills. One example is the use of Dynamic Time Warping (DTW) to classify a time series of surgical data and support feedback provision in lumbar disk herniation surgery \cite{forestier2012classification}. However, this approach is less accurate at the beginning of a procedure when not much data is available. A supervised pattern mining algorithm was used in temporal bone surgery to identify significant behavioural patterns, classified as novice or expert, based on existing examples \cite{zhou2013pattern}. Here, when a novice pattern is detected during drilling, the closest expert pattern was delivered as feedback. However, it is very difficult to identify significant expert/novice patterns as novices and experts often share a large proportion of similar patterns. 

A similar attempt used a prediction model to discriminate the expertise levels using random forests, and then generate feedback directly from the prediction model itself \cite{zhou2013constructive}. Here, the generated feedback was the optimal change that would change a novice level to an expert, based on votes of the random forest (Split Voting (SV)). Decision trees and random forests were used in other research areas as well to provide feedback. For example, a decision tree  based method was used in customer relationship management to change disloyal customers to loyal ones \cite{yang2003postprocessing,yang2007extracting}. Generating feedback from additive tree models such as random forest and gradient boosted trees is NP-hard, but the exact solution can be found by solving a transformed integer linear programming (ILP) problem \cite{cui2015optimal}.

In this paper, we propose a neural network based method to generate feedback using the adversarial technique. One intriguing property of neural networks is that the input can be changed by maximizing the prediction error so that it moves into a different class with high confidence \cite{szegedy2013intriguing}. This property has been used to generate \textit{adversarial examples} from deep neural nets in image classification \cite{goodfellow2014explaining}. An adversarial example is formed by applying small perturbations (imperceptible to the human eye) to the original image, such that the neural network misclassifies it with high confidence. Although the adversarial example has similarities to the feedback problem in that they both change the input to a different class, they are not synonymous. First, the adversarial example is formed by adding intentionally-designed noise that may result in states that do not exist or have practical meaning in a real-world dataset such as that of the feedback problem. Second, only a few changes to inputs are recommended for feedback, to make it useful to follow. These considerations lead to the formal problem definition below.

\section{Problem Definition}
In this section, we discuss the real-time feedback process, and show how skill/behaviour level is defined in SBT applications. We then formally define the feedback generation problem as applied to SBT.

\subsection{Feedback Process Overview}
Figure \ref{fig:fbsystem} illustrates the real-time feedback process in SBT. It operates in two steps: 1) offline training and 2) real-time feedback provision. In the offline stage, a feedback generation method is trained via supervised learning on labelled (novice/expert) skill samples. In real-time, when a trainee is practising on the simulator, novice skill will be captured and input into the feedback generation method to obtain feedback about where improvement is required. Technically, feedback is the suggested action that can improve novice skill to expert skill. Finally, the feedback will be delivered immediately to the trainee to improve behaviour. The focus of this paper is the feedback generation method as highlighted in grey.
% Figure \ref{fig:fbsystem} also illustrates the real-time feedback process.

\begin{figure}[!htb]
\centering
\includegraphics[width=0.48\textwidth]{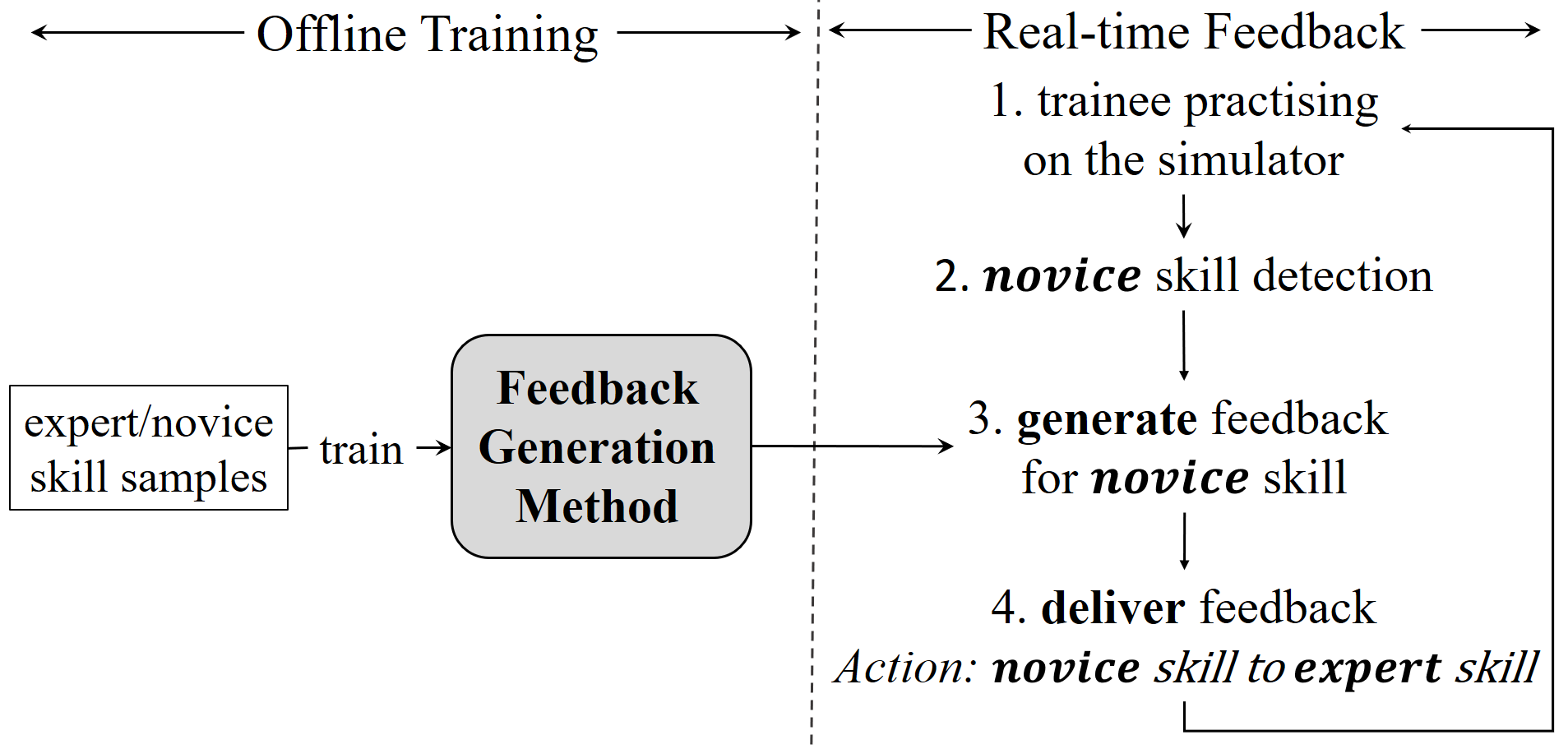}
\caption{The real-time feedback process in SBT.}
\label{fig:fbsystem}
\end{figure}

\subsection{Definition of User Skill}
\label{sec:skill}
SBT often works with multivariate time series data. This is because a SBT task often consists of a series of steps over a period of time. The skill level is usually defined over a period of time, based on the values of certain skill metrics. 

In general, skill metrics are: 1) motion-based, 2) time-based, 3) position-based, or 4) system settings. Motion-based metrics are often signals captured from haptic devices or sensors, for example, the speed and engine rpm in a driving simulator. Time-based metrics measure quantities such as reaction time in military training. Position-based metrics relate to the location of the current procedure and include quantities such as position coordinates and distance to landmarks. System settings refer to measures that affect the environment such as the magnification level in surgical simulation. 

\begin{definition}
\label{def:skill}
In simulation-based training, user skill is a feature vector summarizing user behaviour over an arbitrary period of time and annotated with class labels. 
\end{definition}

For example, consider a SBT environment for surgery. Here, the level of skill can be defined by the quality of a stroke, a continuous motion of the drill with no abrupt changes in direction. The period of time here, over which the user behaviour is summarised, is the time interval of the complete stroke. Metrics that define the quality of a stroke within this time interval include measures such as stroke length, speed, acceleration, duration, straightness, and force \cite{zhou2015automated}. We denote such a skill metric as a \textit{feature}. A vector of feature values that defines user skill is an \textit{instance} and is associated with a class label that denotes the skill level (expert or novice).

\subsection{Feedback Generation Problem}
\label{sec:problem}
Here, we define the feedback generation problem in an expert-novice perspective. We acknowledge that there may be more than 2 levels of expertise in some SBT applications. However, this can be easily addressed using the one-vs-rest approach.

In SBT, the feedback generation problem is to find the optimal action that can be taken to change a novice instance to an expert instance. Suppose the dataset $\mathcal{D}$ consists of $m$ features, $n$ instances defined by the feature vector $\mathbf{x}=(x_1,...,x_m)$, and associated with a class label $y\in\{0,1\}$ (1: expert, 0: novice). $H(\mathbf{x})$ is a prediction model learnt over $\mathcal{D}$. The feedback generation problem can then be defined as follows.

\begin{problem} \label{def:prob}
Given a prediction model $H(\mathbf{x})$ and a novice instance $\mathbf{x}$, the problem is to find the optimal action $A: \mathbf{x} \rightarrow \mathbf{x}_f$ that changes $\mathbf{x}$ to an instance $\mathbf{x}_f$ under limited cost $C$ such that $\mathbf{x}_f$ has the highest probability of being in the expert class:

\label{prob:maximisation}
    \begin{equation*}
        \begin{aligned}
        & \underset{A:\; \mathbf{x} \rightarrow \mathbf{x}_f}{\text{argmax}}
        & & H(\mathbf{x}_f) \\
        & \text{subject to}
        & & \mathscr{C}(\mathbf{x}, \mathbf{x}_f) \leq C,
        \end{aligned}
    \end{equation*}
\end{problem}

\noindent where, feedback $A:\mathbf{x} \rightarrow \mathbf{x}_f$ involves one or multiple feature changes (increase/decrease). For example, $A:(force=\mathbf{0.2}, duration=0.3) \rightarrow (force=\mathbf{0.5},duration=0.3)$ is ``increase $force$ to 0.5". $\mathscr{C}(\mathbf{x}, \mathbf{x}_f)$ is the cost function measuring the potential cost of feedback $A$. In SBT, $\mathscr{C}(\mathbf{x}, \mathbf{x}_f) = \lVert \mathbf{x}-\mathbf{x}_f \rVert_0$, i.e., the number of changed features. The cost limit $C$ is often a small integer such as 1 or 2 in SBT so as to meet the requirements discussed in Section \ref{sec:introduction}.

\section{Proposed Method}
To tackle the feedback generation problem, we propose the use of a neural network as the prediction model $H(\mathbf{x})$ and introduce a method that directly generates feedback from the neural network. Let $H_{\bm{\theta}}(\mathbf{x})$, with parameters/weights $\bm{\theta}$, be the neural network learnt with respect to the loss function $J(\bm{\theta}, \mathbf{x}, y)$, where $\mathbf{x}$ is the input or feature vector, $y$ the class value associated with $\mathbf{x}$, and $y^*$ the target class we want $\mathbf{x}$ to be in.

Recall that during the training process, the weights $\bm{\theta}$ are updated so that the loss $J(\bm{\theta}, \mathbf{x}, y)$ is minimized. Therefore, if we keep $\bm{\theta}$ fixed while the input $\mathbf{x}$ is updated so that $J(\bm{\theta}, \mathbf{x}, y)$ is maximized, we can get a new instance that has high confidence of being in the opposite class to its original class $y$ \cite{szegedy2013intriguing}. To maximize $J(\bm{\theta}, \mathbf{x}, y)$, the input can be updated in the positive direction of the gradient following Equation \eqref{eq:update_x}, where $\epsilon$ is the learning rate. 

\vspace{-2mm}
\begin{equation}
\label{eq:update_x}
\mathbf{x} = \mathbf{x} + \epsilon \nabla_{\mathbf{x}} J(\bm{\theta}, \mathbf{x}, y)
\end{equation}
\vspace{-2mm}

This is the property that has been used to generate adversarial examples in image classification. Since adversarial examples require small perturbations in the input image,  \cite{goodfellow2014explaining} applied a sign function to linearize the loss function around the current value of $\bm{\theta}$ , as shown in Equation \eqref{eq:sign_update_x}. This method updates all the pixels of the input image once to get small perturbations.

\vspace{-2mm}
\begin{equation}
\label{eq:sign_update_x}
\mathbf{x} = \mathbf{x} + \epsilon \; sign(\nabla_{\mathbf{x}} J(\bm{\theta}, \mathbf{x}, y))
\end{equation}
\vspace{-2mm}

Equation \eqref{eq:update_x} works well for two-class tasks. However, for multi-class tasks, there are more than one opposite classes to $y$. This means using Equation \eqref{eq:update_x} cannot guarantee the new instance has high confidence in the target class $y^*$. The alternative is minimizing the loss $J(\bm{\theta}, \mathbf{x}, y^*)$ with respect to the specific target class $y^*$ as defined in Equation \eqref{eq:update_x_2}. 

\vspace{-2mm}
\begin{equation}
\label{eq:update_x_2}
\mathbf{x} = \mathbf{x} - \epsilon \nabla_{\mathbf{x}} J(\bm{\theta}, \mathbf{x}, y^*)
\end{equation}
\vspace{-2mm}

Although Equation \eqref{eq:update_x_2} works for both two-class and multi-class tasks, it still has two potential problems that limit its use for feedback generation in SBT. First, it may change all input features, thus violating the constraint (e.g., $\lVert \mathbf{x}-\mathbf{x}_f \rVert_0 \leq C$) of Problem \ref{def:prob}. Second, the update may explode the values of inputs to extremely small or large values, similar to the \textit{exploding gradient problem} \cite{pascanu2012understanding}. However, in practice, some features may have a certain value range outside of which the feature is meaningless.

To solve the first problem, we introduce a $L1$ regularization term to $J(\bm{\theta}, \mathbf{x}, y^*)$ to control the sparsity of the change so as to generate simple feedback. The new loss function is defined in Equation \eqref{eq:sparsity}, where $\lambda$ is the regularization parameter and $\mathbf{x}_0$ is the original input that needs to be changed.

\vspace{-2mm}
\begin{equation}
\label{eq:sparsity}
J^{'}(\bm{\theta}, \mathbf{x}, y^*) = J(\bm{\theta}, \mathbf{x}, y^*) + \lambda ||\mathbf{x} - \mathbf{x}_0||_1
\end{equation}
\vspace{-2mm}

To solve the second problem, we propose a bounded update approach (see Equations \eqref{eq:update_x_3} and\eqref{eq:update_x_4}) as an alternative to Equation \eqref{eq:update_x_2}. It incorporates the value range (defined by lower and upper bounds) of a feature into the update to ensure the updated feature value is still within range.

\vspace{-2mm}
\begin{equation}
\label{eq:update_x_3}
\mathbf{x} = \mathbf{x} - \epsilon \mathcal{S}_{\mathbf{x}} \Big( \mathbf{x} \mathcal{S}_{\mathbf{x}} - \frac{\mathbf{a}}{2}\big(1+\mathcal{S}_{\mathbf{x}}\big) + \frac{\mathbf{b}}{2}\big(1-\mathcal{S}_{\mathbf{x}}\big) \Big)
\end{equation}
\vspace{-2mm}

\vspace{-2mm}
\begin{equation}
\label{eq:update_x_4}
\mathcal{S}_{\mathbf{x}} = sign(\nabla_{\mathbf{x}} J^{'}(\bm{\theta}, \mathbf{x}, y^*))
\end{equation}
\vspace{-2mm}

\noindent $\mathcal{S}_{\mathbf{x}}$ is the sign of the partial derivative of $J^{'}(\bm{\theta}, \mathbf{x}, y^*)$ with respect to $\mathbf{x}$. The upper and lower bounds of $\mathbf{x}$ are $\mathbf{a}$ and $\mathbf{b}$ respectively, i.e. $x_i \in [a_i,b_i]$. 

According to Equation \eqref{eq:update_x_3}, if the gradient $\nabla_{x_i} J^{'}(\bm{\theta}, \mathbf{x}, y^*)$ is positive (i.e., $\mathcal{S}_{x_i} = 1$), the update will become $x_i = x_i - \epsilon (x_i-a_i)$ which means $x_i$ moves a small step towards its lower bound $a_i$. Similarly, a negative gradient gives $x_i = x_i + \epsilon (b_i-x_i)$, a move towards its upper bound $b$. No update will be applied if the gradient is zero, as in this case, $\mathcal{S}_{x_i} = 0$. This bounded update not only guarantees the correct update direction to minimize the loss, but also ensures that $x_i \in [a_i, b_i]$ always holds true (see Lemma \ref{lemma_1} and proof).

\begin{lemma}
If $a_i \leq x_i \leq b_i$, $0 < \epsilon \ll 1$, and $x^{'}_{i} = x_i - \epsilon \mathcal{S}_{x_i} \big( x_i \mathcal{S}_{x_i} - \frac{a_i}{2}(1+\mathcal{S}_{x_i}) + \frac{b_i}{2}(1-\mathcal{S}_{x_i}) \big)$, then $a_i \leq x^{'}_{i} \leq b_i$.
\label{lemma_1}
\end{lemma}
 
\begin{proof}
The sign function $\mathcal{S}_{x_i}$ only has 3 outputs: 1,0 or -1.
\begin{case}
If $\mathcal{S}_{x_i} = 0$, then $x^{'}_{i} = x_i$.
\end{case}

\noindent In this case, $a_i \leq x^{'}_{i} = x_i \leq b_i$ holds true.

\begin{case}
If $\mathcal{S}_{x_i} = 1$, then $x^{'}_{i} = x_i - \epsilon (x_i-a_i)$.
\end{case}

\noindent In this case, $x^{'}_{i} - a_i = (1-\epsilon)(x_i-a_i)$ and $x^{'}_{i} - b_i = \epsilon(a_i-x_i)+(x_i-b_i)$. Then, $a_i \leq x_i \leq b_i$ and $0 < \epsilon \ll 1$ gives $x_i-a_i \geq 0$, $x_i-b_i \leq 0$ and $1-\epsilon > 0$. Therefore, $x^{'}_i - a_i \geq 0 $ and $x^{'}_i - b_i \leq 0$, that is, $a_i \leq x^{'}_i \leq b_i$. 

\begin{case}
If $\mathcal{S}_{x_i} = -1$, then $x^{'}_i = x_i + \epsilon (b_i-x_i)$.
\end{case}

\noindent And in this case, $x^{'}_i - a_i = \epsilon(b_i-x_i) + (x_i-a_i)$ and $x^{'}_i - b_i = (1-\epsilon)(x_i-b_i)$. 
Similarly, $b_i-x_i \geq 0$, $x_i-a_i \geq 0$ and $1-\epsilon > 0$ gives $x^{'}_i - a_i \geq 0 $ and $x^{'}_i - b_i \leq 0$, that is, $a_i \leq x^{'}_i \leq b_i$. 
\\
\\
\noindent To conclude, in all cases, $a_i \leq x^{'}_i \leq b_i$ holds true.
\end{proof}

Equations \eqref{eq:sparsity}, \eqref{eq:update_x_3} and \eqref{eq:update_x_4} give the definition of the proposed ``neural network-based feedback (NNFB) method". NNFB takes a novice instance $\mathbf{x}$ as input, iteratively updates $\mathbf{x}$ (different from the one-time-update in generating adversarial examples) until it converges or meets the terminating criteria. Let the generated new instance be $\mathbf{x}_f$, the feedback is then the action $A: \mathbf{x} \rightarrow \mathbf{x}_f$ (see the example in Problem \ref{prob:maximisation}).

When feedback is delivered, we need to ensure that it contains only $C$ features. Although, the $L1$ regularization reduces the number of feature changes in general, in the absence of valid feedback with a low number of feature changes, it may still result in ones with higher numbers of feature changes. To overcome this issue, we suggest a post-selection process that iteratively tests all feature changes and select the ones with $C$ or less changes that result in the best improvements.

The proposed method (NNFB) is easily generalizable to different SBT applications. First, the regularization term in $J^{'}(\bm{\theta}, \mathbf{x}, y^*)$ can be adjusted accordingly for different applications (for example, $L2$ norm for applications that prefer small changes). Furthermore, NNFB offers flexible control over feature changes as the lower and upper bounds are adjustable for different features and even for different input instances. For example, we can set $a_i=b_i=x_i$ for a categorical feature that cannot be changed, such as prior simulation experience. This flexibility also benefits those applications that have discrete cost functions as some explicit cost limits can be easily incorporated into the bounds.

\section{Experimental Validation}
\label{sec:experiment}
In this section, we first describe the two real-world datasets that were used in the experiments. Then, we briefly introduce the existing methods that the proposed method was compared against, followed by the experimental setup. Finally, we discuss the experiment results.

\subsection{Datasets}
We tested our method on two real-world SBT datasets. These datasets were collected from a temporal bone surgical simulator designed to train surgeons in ear-related surgeries. 7 expert and 12 novice surgeons performed two different surgeries that require very different surgical skills: cortical mastoidectomy - dataset 1 ($\mathcal{D}1$)  and posterior tympanotomy - dataset 2 ($\mathcal{D}2$). Surgical skill is defined by 6 numeric skill metrics: stroke length, drill speed, acceleration, time elapsed, the straightness of the trajectory and drill force (see example in Section \ref{sec:skill} for more details). The skill metrics were recorded by the simulator at a rate of approximately 15 Hz. Overall, $\mathcal{D}1$ includes 60K skill instances (28K expert and 32K novice) while $\mathcal{D}2$ includes 14K skill instances (9K expert and 5K novice). Both datasets were normalized to the range of $[0,1]$ using feature scaling as follows. 

\begin{equation}
\label{eq:scaling}
    \mathbf{x}^{'} = \frac{\mathbf{x}-\mathbf{x}_{min}}{\mathbf{x}_{max} - \mathbf{x}_{min}}
\end{equation}

\noindent where, $\mathbf{x}$ and $\mathbf{x}^{'}$ are the original and scaled feature vectors respectively and $\mathbf{x}_{min}$ and $\mathbf{x}_{max}$ are the minimum and maximum feature values of $\mathbf{x}$ respectively.

\subsection{Compared Methods}
\setcounter{footnote}{0}
Existing feedback generation methods compared with NNFB are as follows. 
\begin{itemize}
\item Split Voting (SV):
This is the random forest based state-of-the-art generation method for providing real-time feedback \cite{zhou2013constructive} as discussed in Section \ref{sec:related}.
\item Integer Linear Programming (ILP):
This method solves the random forest feedback generation problem by transforming it to an integer linear programming problem \cite{cui2015optimal} as discussed in Section \ref{sec:related}.
\item Random Iterative (RI):
This method randomly selects a feature and iteratively selects the best value among the feature's value partitions in the random forest \cite{cui2015optimal}.
\item Random Random (RR):
This method randomly picks a feature from a novice instance and selects a random change to that feature as the suggested feedback.
\end{itemize}

\subsection{Experimental Setup}
\label{sec:setups}

For testing, we randomly chose one novice participant, then took all instances performed by this novice as the test set. The remaining instances were used for training. This simulates the real-world scenario of an unknown novice using the simulator. Parameter tuning was performed on the training data based on a 11-fold leave-one-novice-out cross-validation. In each fold, we took all instances from one randomly chosen novice as the validation set.

All methods were restricted to generate feedback with only one feature change, which is a typical requirement in SBT. This is a binary task as there are only 2 skill levels (\textit{expert} and \textit{novice}). All methods were then evaluated using 2 measures: 1) efficiency and 2) effectiveness. Overall, a good feedback generation method should have high effectiveness and high efficiency (low time-cost).

Efficiency was measured using the time-cost (in seconds) spent on average to generate feedback for one novice instance. The novice instance $\mathbf{x}$ will be changed to the target instance $\mathbf{x}_f$ by the feedback $A: \mathbf{x} \rightarrow \mathbf{x}_f$ (see Section \ref{sec:problem}). Thus, we use the quality of the target instances (i.e., $\{\mathbf{x}_f\}$) to measure the effectiveness of the feedback. As defined in Equation \eqref{eq:effectiveness}, effectiveness $\mathscr{E}$ is the percentage of expert instances in $\{\mathbf{x}_f\}$. 

\begin{equation}
\label{eq:effectiveness}
\mathscr{E} = \frac{|\{\ \mathbf{x}_f| \, \mathbf{x}_f \:is \: an\: expert\: instance\}|}{|\{\mathbf{x}_f \}|}
\end{equation}

However, how instances are classified is dependent on the classifier used. To obtain more convincing results, we used 6 classifiers of different types for evaluation. The evaluation classifiers are: neural network (NN), random forest (RF), logistic regression (LR), SVM (RBF kernal), naive Bayes (NB) and KNN ($K=10$). A generation method that scores consistently high levels of $\mathscr{E}$ across classifiers is deemed effective.

Experiments were carried out on a typical PC with 2.40GHz CPU. The ILP solver used for the ILP method was CPLEX\footnote{https://www-01.ibm.com/software/commerce/optimization/cplex-optimizer} as suggested by the authors, and the neural network/random forest implementations we used were from scikit-learn. Default settings in scikit-learn were used for parameters not specifically mentioned here.

\subsection{Parameter Tuning}
Parameter tuning was performed on the training data with a 11-fold leave-one-novice-out cross-validation as mentioned above in Section \ref{sec:setups}. A two-layer neural network architecture was used for NNFB. For $\mathcal{D}1$, a neural network with 250 hidden neurons was selected for NNFB while a random forest with 120 trees was selected for SV, RI and ILP. For $\mathcal{D}2$, NNFB used a neural network with 120 hidden neurons while SV, RI and ILP used a random forest with 100 trees. These parameters were selected based on the turning point of the number of hidden neurons or the number of trees with respect to the mean squared error (MSE) of the neural network and random forest respectively. 

\begin{figure}[ht]
\centering
\includegraphics[width=0.48\textwidth]{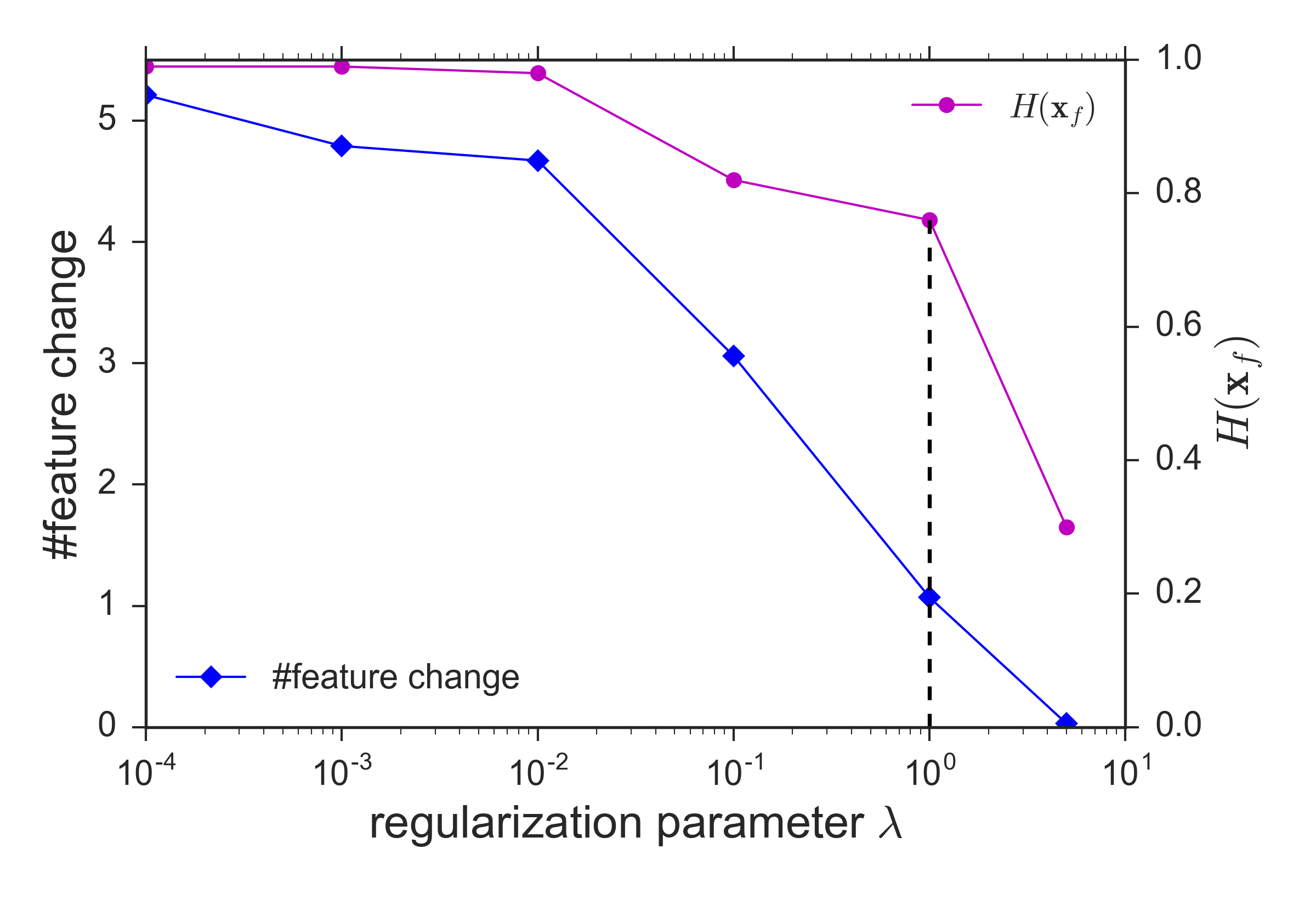}
\caption{ With the increases of $\lambda$, the number of feature changes in the feedback decreases but the confidence of $\mathbf{x}_f$ being in the expert class ($H(\mathbf{x}_f)$) remains high.}
\label{fig:n_change}
\end{figure}

\begin{table*}[t]
\renewcommand{\arraystretch}{1.5}
\caption{The effectiveness (mean$\pm$std) tested by 6 evaluation classifiers. The best results are highlighted in bold.}
\label{table:effect_compare}
\centering
\begin{tabular}{c|c|cccccc}
\hline
% \cline{3-8}
 \multicolumn{2}{c|}{}  & NN & RF & LR & SVM & NB & KNN \\ \hline
\multirow{5}{*}{$\mathcal{D}1$} & RR & 0.19$\pm$0.06 & 0.23$\pm$0.10 & 0.35$\pm$0.07 & 0.27$\pm$0.06 & 0.32$\pm$0.12 & 0.30$\pm$0.05 \\ 
% IT & 0.98$\pm$0.04 & 1.00$\pm$0.00 & 0.87$\pm$0.08 & 0.96$\pm$0.04 & 0.76$\pm$0.12 & 0.76$\pm$0.08 \\ \hline
& RI & 0.44$\pm$0.07 & 0.39$\pm$0.04 & 0.50$\pm$0.08 & 0.46$\pm$0.06 & 0.42$\pm$0.12 & 0.40$\pm$0.08 \\ 
& SV & 0.63$\pm$0.07 & 0.59$\pm$0.06 & 0.60$\pm$0.07 & 0.62$\pm$0.06 & 0.50$\pm$0.11 & 0.53$\pm$0.07 \\ 
& ILP & 0.72$\pm$0.04 & \textbf{0.87$\pm$0.00} & 0.71$\pm$0.05 & 0.76$\pm$0.04 & \textbf{0.70$\pm$0.11} & \textbf{0.76$\pm$0.04} \\ 
& NNFB & \textbf{0.86$\pm$0.01} & 0.82$\pm$0.08 & \textbf{0.78$\pm$0.05} & \textbf{0.82$\pm$0.04} & 0.68$\pm$0.14 & 0.73$\pm$0.08 \\ \hline

\multirow{5}{*}{$\mathcal{D}2$} & RR & 0.21$\pm$0.04 & 0.22$\pm$0.07 & 0.29$\pm$0.04 & 0.37$\pm$0.02 & 0.32$\pm$0.11 & 0.32$\pm$0.06 \\ %\cline{2-8}
% IT & 0.98$\pm$0.04 & 1.00$\pm$0.00 & 0.87$\pm$0.08 & 0.96$\pm$0.04 & 0.76$\pm$0.12 & 0.76$\pm$0.08 \\ \hline
& RI & 0.48$\pm$0.04 & 0.49$\pm$0.04 & 0.47$\pm$0.09 & 0.52$\pm$0.05 & 0.47$\pm$0.12 & 0.43$\pm$0.10 \\ 
& SV & 0.61$\pm$0.08 & 0.69$\pm$0.04 & 0.62$\pm$0.05 & 0.61$\pm$0.07 & 0.56$\pm$0.11 & 0.59$\pm$0.04 \\ 
& ILP & 0.88$\pm$0.04 & \textbf{0.90$\pm$0.02} & 0.79$\pm$0.07 & \textbf{0.77$\pm$0.03} & 0.78$\pm$0.12 & \textbf{0.84$\pm$0.09} \\ 
& NNFB & \textbf{0.92$\pm$0.02} & 0.82$\pm$0.06 & \textbf{0.81$\pm$0.07} & 0.72$\pm$0.05 & \textbf{0.79$\pm$0.11} & 0.81$\pm$0.07 \\ \hline
\end{tabular}
\end{table*}

In terms of the regularization parameter $\lambda$ in NNFB, Figure \ref{fig:n_change} indicates that larger $\lambda$ results in simple feedback with a fewer number of feature changes. When $\lambda=1$, a feedback on average consists of only one feature change, but remains highly confident ($H(\mathbf{x}_f)>0.7$) to change a novice instance to an expert instance. Therefore, we chose $\lambda=1$ for NNFB. Since datasets have been normalised, the upper bounds for all features are 1 and the lower bounds are 0. Other settings for NNFB include Rectified Linear Unit (ReLU) activation function \cite{glorot2011deep}, cross entropy loss and learning rate $\epsilon=\num{0.0001}$.

\subsection{Results}

We first demonstrate the overall performance considering both effectiveness and efficiency. Figure \ref{fig:nnfb_performance} illustrates the effectiveness of each method as evaluated using 6 different evaluation classifiers with respect to the time-cost (inverses to efficiency) for each dataset. As seen in the figure, the proposed method shows the desired performance of highest effectiveness at an acceptably low time-cost (within the real-time time-limit) when compared to the other methods. This proves that the adversarial technique can be used to generate effective and timely feedback for SBT. Note that the slightly higher variance of the NNFB method indicates the varying resistance of test classifiers to adversarial generation.

\begin{figure}[ht!]
\centering
\includegraphics[width=0.48\textwidth]{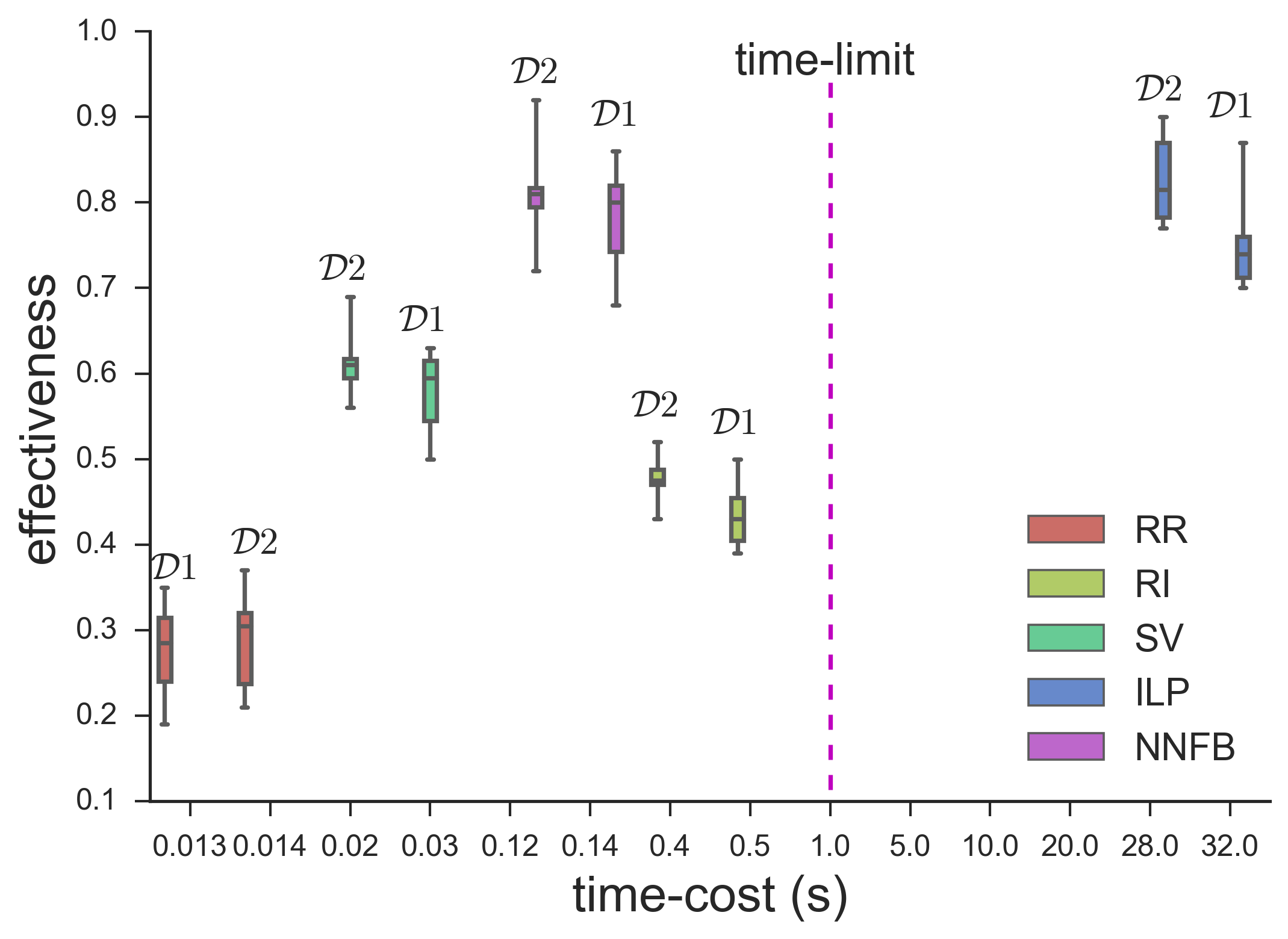}
\caption{Box plot representing the performance of the 6 evaluation classifiers with respect to effectiveness and time-cost. Each method has two boxes that represent the 2 datasets $\mathcal{D}1$ and $\mathcal{D}2$. Colored view is recommended.} 
% This is a box plot of the effectiveness distribution vs. time-cost (inverses to efficiency). The quartile lines divide the distribution into 4 quartile groups with each represent $25\%$ of the distribution. The diamonds are outliers. Each method has two boxes w.r.t. dataset $\mathcal{D}1$ and $\mathcal{D}2$. The best method can be found in the desired quadrant defined by the dash coordinates. Colored view is recommended.
\label{fig:nnfb_performance}
\end{figure}

Detailed results for effectiveness of the feedback generation methods across the 6 evaluation classifiers is shown in Table \ref{table:effect_compare} for the datatsets $\mathcal{D}1$ and $\mathcal{D}2$. On both datasets, NNFB achieved comparable performance to ILP and outperformed all others methods across all classifiers. However, as shown in Table \ref{table:efficiency}, ILP violates the real-time time-limit as discussed in Section \ref{sec:related}, and as such, will not be suitable for most SBT applications. Although both RR and SV are more efficient than the proposed method in feedback generation, they show significantly lower levels of effectiveness when compared to NNFB. Thus, it can be concluded that in terms of both effectiveness and efficiency, the proposed method is the best suited for providing real-time feedback in SBT applications.

\begin{table}[ht]
\renewcommand{\arraystretch}{1.5}
\caption{The time-cost (mean$\pm$std in seconds) for generating one feedback, tested on datasets $\mathcal{D}1$ and $\mathcal{D}2$.}
\label{table:efficiency}
\centering
\begin{tabular}{c|cc}
\hline
  & $\mathcal{D}1$ & $\mathcal{D}2$ \\ \hline
RR & 0.013$\pm$0.004 & 0.014$\pm$0.001\\ 
RI & 0.504$\pm$0.098 & 0.401$\pm$0.020\\ 
SV & 0.023$\pm$0.003 & 0.017$\pm$0.003 \\ 
ILP & 31.738$\pm$2.439 & 27.760$\pm$3.107 \\ 
NNFB & 0.142$\pm$0.029 & 0.121$\pm$0.016\\ \hline
\end{tabular}
\end{table}

\section{Conclusion and Discussion}
In this paper, we introduced a technique for the adversarial generation of real-time feedback with neural networks for SBT. The proposed method (NNFB) applies a bounded adversarial update on the novice skill vector to generate an optimal expert skill vector in order to be used in the provision of feedback. To ensure that the suggested action is simple enough to practically undertake, we adopted $L1$ regularization to obtain feedback with a fewer number of feature changes. We explored theoretically the validity of NNFB, and showed empirically that it outperforms existing methods in providing effective real-time feedback.

Improving human performance in practice is a very challenging task. It involves many aspects of the learning process such as the learning environment, the task complexity, the knowledge level of the leaner, and the feedback intervention. In the future, we will deploy the proposed method to SBT environments and conduct user studies with human experts to further validate the method and investigate its effectiveness in teaching skills in practical applications.

\section*{Acknowledgements}

The authors would like to thank the US Office of Naval Research for funding this project.

\bibliographystyle{named}
\bibliography{ijcai17}

\end{document}